\documentclass[10pt, conference, compsocconf]{IEEEtran}

\usepackage{amsmath,graphicx,amssymb,amsfonts,amsthm}
\usepackage{enumerate}
\usepackage{bm}
\usepackage{balance}

\newtheorem{theorem}{Theorem}

\newtheorem{corollary}{Corollary}

\def\T{{ \mathrm{\scriptscriptstyle T} }}

\begin{document}
%
\title{Learning the Number of Autoregressive Mixtures in \\ Time Series Using the Gap Statistics}

\author{\IEEEauthorblockN{Jie Ding, Mohammad Noshad and Vahid Tarokh}\\
\IEEEauthorblockA{John A. Paulson School of Engineering and Applied Sciences\\
Harvard University \\
Cambridge, MA 02138, USA\\
Email: jieding@g.harvard.edu, mnoshad@seas.harvard.edu, vahid@seas.harvard.edu}
}


\markboth{}{Shell \MakeLowercase{\textit{et al.}}: Bare Demo of IEEEtran.cls for Journals}

\maketitle
\thispagestyle{empty}
\pagestyle{empty}

\begin{abstract}
Using a proper model to characterize a time series is crucial in making accurate predictions. In this work we use time-varying autoregressive process (TVAR) to describe non-stationary time series and model it as a mixture of multiple stable autoregressive (AR) processes.
We introduce a new model selection technique based on Gap statistics to learn the appropriate number of AR filters needed to model a time series. We define a new distance measure between stable AR filters and draw a reference curve that is used to measure how much adding a new AR filter improves the performance of the model, and then choose the  number of AR filters that has the maximum gap with the reference curve.
To that end, we propose a new method in order to generate uniform random stable AR filters in root domain.
Numerical results are provided demonstrating the performance of the proposed approach.
\end{abstract}

\section{Introduction}
\label{sec:Intro}
Modeling time-series has been of a great interest for a long time.
A good model not only describes the observed data well, but also avoids over-fitting, which can reduce the predictive power of the model.
Autoregressive (AR) models are one of the most commonly used techniques to model stationary time-series \cite{Time-Series-Anal-94}.
A time-varying autoregressive (TVAR) model is a generalized form of an AR model that is used to describe the non-stationarity of time-series \cite{TVAR-07}.
%
An example of TVAR models is the regime-switching model \cite{Time-Series-Anal-94}, which assumes that a non-stationary stochastic process is composed of different epochs/regimes each of which is a stationary process, and that the regimes switch according to a Markov process. Another example is the model proposed in \cite{AR-Mixture-00}, which uses a mixture of Gaussian AR models to describe time-series and uses an expectation-maximization (EM) algorithm to determine the parameters of the model. 
Wong and Li \cite{AR-Mixture-00} used Akaike information criterion (AIC) \cite{AIC-73} and Bayes information criterion (BIC) \cite{BIC-78} to introduce a penalty on the complexity of the model and estimate the number of AR filters. In general AIC and BIC are shown to be suboptimal for determining the number of modes \cite{Problem-of-AIC-BIC-96}. 

Tibshirani \emph{et al}. \cite{Gap-Statistics-01} introduced an intuitive technique for determining the appropriate number of clusters using Gap statistics. The general idea of the Gap statistics is to identify the number of clusters in a data set by comparing the goodness of fit for the observed data with its expected value
under a reference distribution. In this work we extend the Gap statistics to time-series in order to identify the number of AR filters needed to describe a time-series.
We use a reference curve to measure how much adding a new AR filter improves the model under reference distribution, and then choose the number of filters that has the maximum gap with that reference curve.

In order to derive a reference curve, it is important to first clarify the meaning of the ``goodness of fit''.
In \cite{Gap-Statistics-01}, the ``goodness of fit'' is measured by the sum of squared Euclidean distances of each data point from the center of the cluster it belongs to. But a different measure needs to be used in time-series to evaluate the performance of each model. In this work, our goal is to select models that have higher predictive powers, and thus we define the ``goodness of fit'' measure to be the mean squared prediction error (MSPE). We use MSPE to define a new distance measure between two stable filters in accordance with our need for a reference curve in Gap statistics. Our proposed distance measure differs from the previous distances (e.g., cepstral distance \cite{Dist-for-ARMA-01}, discrete Fourier transform (DFT) \cite{Dis-DFT-94}, principal component analysis (PCA) \cite{Dis-PCA-00}, and discrete wavelet transform (DWT) \cite{Dis-DWT-99}) in that it naturally arises from the MSPE of the one-step prediction.

Computing each point on the reference curve using the new distance measure turns out to be a clustering problem in the space of stable AR filters with a fixed size, which is solved using the Monte Carlo approach. To that end, we introduce an approach to generate uniform random stable filters with equal sizes, and apply the $k$-medoids algorithm 
to approximate the optimal solution for the clustering problem. Numerical simulations show that the accuracy of the proposed Gap statistics in estimating the number of AR filters surpasses that of AIC and BIC.




The remainder of the paper is organized as follows. Section~\ref{sec:Model} describes the model considered in this paper.
 In Section~\ref{sec:Gap} we propose the Gap statistics to estimate the number of AR filters in a time-series. Section~\ref{sec:performance} presents some numerical results to evaluate the performance of the proposed approach. We make our conclusions in Section~\ref{sec:conclusion}.

\section{Model Assumption and Its Estimation}
\label{sec:Model}
In order to model a given time series $\mathbf{X}=\{x_n\}_{n=1}^N$, we assume that each data point depends only on $L$ previous points and that $L$ is known in this paper. We use a time-varying autoregressive (TVAR) model to describe the value at time step $n$ as follows
    \begin{align}
        x_n = \sum_{\ell=1}^{L} \phi_{n\ell} x_{n-\ell}  + \varepsilon_n,
    \end{align}
where $\phi_{n\ell}$'s are real numbers, and $\varepsilon_n \sim \mathcal{N}(0,\sigma^2)$ are independent  Gaussian noise.
Assume that the first $L$ points of a sample set, $x_{0}, \cdots, x_{1-L}$, are known.
The vector form of this equation can be written as
    \begin{align}
        x_n = \bm \phi_n^\T \bm x_n + \varepsilon_n,
    \end{align}
where $\bm \phi_n = [\phi_{n1},\phi_{n2},\cdots,\phi_{nL}]^\T $, and $\bm x_n = [x_{n-1}$, $x_{n-2}$ $\cdots$ $,x_{n-L}]^\T$. In real scenarios, $\bm \phi_n$ is a time-varying vector, and modeling the variations of $\bm \phi_n$ can be complicated.
For simplicity, we assume that $\bm \phi_n \in \{\bm \gamma_1,\dots,\bm \gamma_M \}$, for $n=1,2,\cdots,N$, where $M$ is the number of AR filters used to describe $\mathbf{X}$, and that each $\bm \gamma_m$ is a filter with size $L$. 
We refer to $\bm \gamma_m$ as mode $m$ and call this model multi-mode AR model.

Clearly,  $M$ can be seen as a parameter for a nested family of models, and larger $M$ will fit the observed data better.  But as mentioned before, the predictive power of the model drops if $M$ is too large. Hence, a model selection procedure that identifies the appropriate number of modes is important.
To that end, we evaluate MSPE for a range of values of $M$ which is assumed to contain the number of modes. We first estimate the parameters of each of the candidate models, and then select the number of AR modes  according to the Gap statistics developed in Section~\ref{sec:Gap}.
For simplicity, we further assume the model with the following log-likelihood:
%
\begin{align}
	\log p(\mathbf{X} |  \Theta) &= \sum\limits_{n=1}^N \log p\left(x_{n} \mid \bm x_n \right) \nonumber \\
	&= \sum\limits_{n=1}^N \log \left( \sum_{m=1}^M \alpha_m \mathcal{N} (x_n \mid \bm \gamma_m^\T \bm x_n, \sigma^2)  \right)  \label{lik}
\end{align}
where $\sum_{m=1}^M \alpha_m = 1, \alpha_m \geq 0$ for any fixed $M$, and $\mathcal{N}(x \mid \mu,\sigma^2)$ denotes the density of Gaussian distribution of mean $\mu$ and variance $\sigma^2$ evaluated at $x$, i.e., $\mathcal{N}(x \mid \mu,\sigma^2) = (2\pi)^{-1/2}\sigma^{-1}\exp\{-(x-\mu)^2/(2\sigma^2)\}$.





Let $\Theta = \{\alpha_m, \bm \gamma_m, \sigma^2 | m=1,\cdots,M \}$ be the set of unknown parameters to be estimated.
Next, we briefly describe how to approximate the maximum-likelihood estimation (MLE) of  $\Theta$.
Though computing the MLE in (\ref{lik}) is not tractable, it can be approximated by a local maximum via the EM algorithm \cite{dempster1977maximum}.
Let $\mathbf{Z} = \{\bm z_n\}_{n=1}^{N}$ be the membership labels with $\bm z_n = [z_{n1}, z_{n2},\cdots, z_{nM}]^\T$, where
$$
z_{nm}=\left\{
\begin{aligned}
1 & \textrm{ if $\bm \phi_n = \bm \gamma_m$} \\
0 & \textrm{ otherwise}
\end{aligned}
\right..
$$
The joint probability of $\mathbf{X}$ and $\bm{Z}$ can be written as a product of conditional probabilities as
\begin{align*}
	p(\mathbf{X}, \bm Z\mid \bm \Theta) =  \prod\limits_{n=1}^N  \prod\limits_{m=1}^M \left( \alpha_m p(x_n \mid \bm x_n ) \right)^{z_{nm}}.
\end{align*}
Thus the complete log-likelihood is 
\begin{align} \label{complete_lik}
	\log p(\mathbf{X}, \mathbf{Z}\mid  \Theta) =  \sum\limits_{n=1}^N  \sum\limits_{m=1}^M z_{nm} \log \left( \alpha_m p(x_n \mid \bm x_n ) \right).
\end{align}
The EM algorithm produces a sequence of estimates by recursive application of E-step and M-step until some convergence criterion is achieved.
For brevity, we provide the EM  formulas below without detailed derivations. We note that the M-step uses a coordinate ascent algorithm to find a local maxima.
\\{\bf E-Step:} We take the expectation of (\ref{complete_lik}) with respect to the missing data $\bm Z$ given the recent estimated unknown parameters, and obtain the following function (also referred to as the ``Q function'')   
\begin{align*}
	Q(\Theta \mid \bm X, \bm Z, \Theta^{\textrm{old}}) =  \sum\limits_{n=1}^N  \sum\limits_{m=1}^M w_{nm} \log \left( \alpha_m p(x_n \mid \bm x_n ) \right),
\end{align*}
where   
\begin{align*}
	w_{nm} =  \frac{\alpha_m \mathcal{N}(x_n \mid \bm \gamma_m^\T \bm x_n, \sigma^2) }{\sum\limits_{m'=1}^M \alpha_{m'} \mathcal{N}(x_n \mid \bm \gamma_{m'}^\T \bm x_n, \sigma^2)}.
\end{align*}
We note that the parameters involved in the right-hand side of the  equation above take values from the last update.
The 	``old'' superscriptions are omitted for brevity.
In other words, the E-step replaces the ``missing data'' $z_{nm}$   in Equation (\ref{complete_lik}) by its expected values $w_{nm} $ 
for $n=1,\cdots, N, \, m = 1,\cdots, M$.
\\{\bf M-Step:}  Letting the derivatives of the Q function be zero leads to a coupled non-linear  system
that has no closed-form solution. 
Thus, our best hope is an approximation to the solution; for this we use the coordinate ascent algorithm to obtain a local maximum. 
For each $\alpha_m $, we apply the Lagrange method with constraint $\sum_{m=1}^M \alpha_m =1 $ to obtain 
\begin{align*}
	\alpha_m &=  \frac{ \sum\limits_{n=1}^N w_{nm} }{\sum\limits_{m=1}^M \sum\limits_{n=1}^N w_{nm} } = \frac{ \sum\limits_{n=1}^N w_{nm} }{N},
	\; m= 1,\cdots, M.
\end{align*}
%
%
Taking the partial derivative of the Q function with respect to $\gamma_m$  and then $\sigma^2$ we obtain  the following local maximum 
\begin{align*}
	\bm \gamma_m &=  \left( \sum\limits_{n=1}^N w_{nm} \bm x_n \bm x_n^\T  \right)^{-1} \left( \sum\limits_{n=1}^N w_{nm} x_n \bm x_n \right),
	\, m= 1,\cdots, M, \\
	\sigma^2 &=   \frac{ \sum\limits_{m=1}^M\sum\limits_{n=1}^N w_{nm} \left( x_n-\bm \gamma_m^\T \bm x_n \right)^2 }{N}.
\end{align*}

\section{Gap Statistics to Determine the Number of Modes}
\label{sec:Gap}
In this work we use Gap statistics \cite{Gap-Statistics-01} to estimate the number of AR filters in a time series. In this technique, a data set $\mathfrak{B}=\{y_1,y_2,\cdots,y_F\}$ is clustered into $M$ disjoint clusters $C_1,C_2,\cdots,C_M$
%
by minimizing the following within-cluster sum of distances (WCSD)
\begin{align} \label{objective_general}
      W_M \triangleq   \min\limits_{\mu_1,\cdots, \mu_M} \left\{ \sum_{m=1}^M \sum_{y \in C_m} d(\mu_m,y) \right\} ,
 \end{align}
 where $d(\mu_m,y)$ is the distance of $y$ from $\mu_m$, the center of the $m$th cluster. Each cluster $C_m$ is defined based on $\mu_1,\cdots,\mu_M$ as 
 \begin{align} \label{clusters}
 C_m = \{y  : y \in \mathfrak{B}, \, d(\mu_m,y) \leq  d(\mu_{m'},y) \ \ \forall m' \neq m  \}. 
 \end{align}
After computing $W_M$ for $M=1,2,\cdots,M_{\text{max}}$ where $M_{\text{max}}$ is assumed to be the largest possible number of clusters, the graph of $\log(W_M)$ is standardized by comparing it to its expectation under a non-informative reference distribution.
This can be chosen to be a uniform distribution. 
The point that has the largest difference with the reference curve is selected as the estimated number of clusters.

Let $d(y,\mu_m)$ be the squared Euclidean distance (the most commonly used distance measure for clustering purposes).
Then with this distance, $W_M$ becomes within-cluster sum of squares. However, for clustering AR filters, Euclidean distances have been shown to be ineffective \cite{Dist-for-ARMA-01}. Hence, we introduce a new distance measure in Sec. \ref{sec:distance} that is well-suited for AR clustering.


\subsection{Distance Measure for Autoregressive Processes}
\label{sec:distance}
In order to find a reference curve for Gap statistics,  we derive the distance between two filters based on MSPE.
We assume that the data is generated by a stable filter $\bm \psi_A = [\psi_{A1},\psi_{A2},\cdots,\psi_{AL}]^\T$ with size $L$, i.e.,
    \begin{align}
        x_n = \bm \psi_A^\T \bm x_n + \varepsilon_n, \quad \varepsilon_n &\sim  \mathcal{N} (0, \sigma^2).
    \end{align}
For now we assume that $x_n$ has zero mean. 
Let $\Psi_A(z) = \sum\limits_{\ell=1}^{L}\psi_{A\ell} z^{-\ell}$ be the characteristic polynomial of $\bm \psi_A$, and $a_1, \cdots, a_{L}$ denote the roots of $1-\Psi_A(z)$, i.e.,
    \begin{align}
        1- \Psi_A(z) = \prod\limits_{\ell=1}^{L} \left( 1-a_\ell z^{-1} \right).
    \end{align}
If  $\bm \psi_A$ is stable, the roots $a_1, \cdots, a_{L}$ lie inside the unit circle ($|a_{\ell}| < 1$). When we use $\bm \psi_A$ at time step $n-1$ to predict the value at time $n$, i.e., $\hat{x}_{n}=\bm \psi_A^\T \bm x_n$, the MSPE is
    \begin{align} \label{Pred Error of Filter A}
        E\{ (x_{n} - \hat{x}_{n})^2 \} = \sigma^2.
    \end{align}
Suppose that we use a filter other than $\bm \psi_A$ to predict the value at time $n$. The mis-specified filter is denoted by $\bm \psi_B=[\psi_{B1},\cdots,\psi_{BL}]^\T$. Then the MSPE becomes
    \begin{align} \label{Pred_Error_of_FilterB}
        E\{ (x_{n} - \hat{x}_{n})^2 \} = E\left\{ \left((\bm \psi_A^\T - \bm \psi_B^\T) \bm x_n \right)^2 \right\} + \sigma^2.
    \end{align}
Motivated by Equations (\ref{Pred Error of Filter A}) and (\ref{Pred_Error_of_FilterB}),
we define the distance between filters $\bm \psi_A$ and $\bm \psi_B$ by
    \begin{align} \label{}
        D\left(\bm \psi_A , \bm \psi_B \right) =  E\left\{ \left((\bm \psi_A^\T - \bm \psi_B^\T) \bm x_n\right)^2 \right\},
    \end{align}
which can be calculated using the power spectral density of $x_n$ as
    \begin{align} \label{def_distance}
        D\left(\bm \psi_A , \bm \psi_B \right) =  \frac{\sigma^2}{2 \pi} \int_{-\pi}^{\pi} \frac{\left| \Psi_A (e^{-j\omega} )-\Psi_B (e^{-j\omega} )\right|^2}{ \left|1-\Psi_A (e^{-j\omega} )\right|^2}  d\omega.
    \end{align}
Using $z=e^{j\omega}$ ($j =\sqrt{-1}$), we get
    \begin{align} \label{Dist in terms of z integral}
        D&\left(\bm \psi_A , \bm \psi_B \right) \nonumber\\&=  \frac{\sigma^2}{2 \pi j} \oint_C \frac{\left( \Psi_A (z )-\Psi_B (z )\right)\left( \Psi_A (z^{-1} )-\Psi_B (z^{-1} )\right)}{ \left(1-\Psi_A (z)\right)  \left(1-\Psi_A (z^{-1})\right)} \frac{ dz}{z}.
    \end{align}
Using Cauchy's integral theorem, Equation (\ref{Dist in terms of z integral}) can be written in terms of the roots $a_1,\cdots,a_{L}$ and $b_1,\cdots,b_{L}$ as
    \begin{align} \label{Dist}
        D\left(\bm \psi_A , \bm \psi_B \right) =
        \sigma^2 \sum_{k=1}^{L} \frac{\prod\limits_{\ell=1}^{L} (a_k-b_{\ell})}{a_k \prod\limits_{\substack{\ell=1\\ \ell \neq k}}^{L} (a_k-a_{\ell})} \left( \frac{\prod\limits_{\ell=1}^{L} (1- a_k b_{\ell}^{*})}{\prod\limits_{\ell=1}^{L} (1- a_k a_{\ell}^{*})} -1 \right).
    \end{align}
for $a_k\neq 0, a_k \neq a_{\ell}, k \neq \ell$, where $a^{*}$ is the conjugate of a complex number  $a$.
For the degenerate cases when $a_k=0$ or $a_k = a_{\ell}$, $D(\bm \psi_A , \bm \psi_B )$ reduces to $\lim\limits_{a_k \rightarrow 0} D(\bm \psi_A , \bm \psi_B )$
or $\lim\limits_{a_k \rightarrow a_{\ell}} D(\bm \psi_A , \bm \psi_B )$.
We note that the distance (\ref{Dist}) is not symmetric, i.e., $D\left(\bm \psi_A , \bm \psi_B \right)  \neq D\left(\bm \psi_B , \bm \psi_A \right)$.
%
The distance measure defined in (\ref{Dist}) is proportional to $\sigma^2$, which results in a constant $\log \sigma^2$ in the computation of $\log W_M$ for the reference curve. Since it is the same for different $M$, without loss of generality we can set $\sigma^2=1$.

\subsection{ Generation of Uniformly Distributed Random Filters  }
\label{sec:distribution}
As mentioned before, Gap statistics require a reference curve that is generated by clustering sampled data from a reference distribution, which is usually chosen to be uniform.
Therefore, uniform random generation of stable filters (uniform in coefficient domain) is needed.
Since the roots are used in Equation (\ref{Dist}), we use an approach to generate samples of roots that correspond to uniform samples of coefficients.
%
To simplify the notations, we let $\Lambda(z)= z^L+ \sum\limits_{\ell=1}^{L} \lambda_{\ell} $ denote the polynomial $z^L (1-\Psi(z) )= z^L - \sum\limits_{\ell=1}^{L}\psi_{\ell} z^{L-\ell} $. 

A polynomial is stable (also referred to as Schur stable) if all of its roots lie inside the unit circle. For a polynomial of order $L$, we use $c \le \lfloor L/2 \rfloor$ and $r$ respectively to denote its number of pairs of complex roots and number of real roots. Let 
$$R_L = \{(\lambda_1, \cdots, \lambda_{L} )\mid z^L + \sum\limits_{\ell=1}^{L}\lambda_{\ell} z^{L-\ell} \textrm{ is stable} \} \subset \mathbb{R}^L$$ 
be the coefficient space of all stable polynomials of degree $L$, and let $R_L^{(c)} \subset R_L$ correspond to the polynomials that have $c$ pairs of complex roots. We call $R_L^{(c)}$ a configuration of $R_L$ with parameters $(L,c)$. Clearly, $R_L$ are bounded subspaces of $\mathbb{R}^L$ and $R_L = R_L^{(0)} \cup \cdots \cup R_L^{(\lfloor L/2 \rfloor)}$.

In this section we propose an approach to generate uniformly distributed  polynomials using roots. We first present the following theorem, which helps us to find the relation between the distribution of coefficients of a polynomial and its roots.
\begin{theorem} \label{Thm:Jac-Disc}
The determinant of the Jacobian matrix of the coefficients of a polynomial with respect to its roots is the Vandermonde polynomial of the roots, i.e.,
    \begin{align} \label{Jacob1}
            \det \left( \frac{\partial [\lambda_1,\cdots,\lambda_L]}{\partial [a_1,\cdots,a_L]} \right) =   \prod\limits_{1 \leq u < v \leq L} (a_v-a_u) .
    \end{align}
Furthermore, the volume of $R_L^{(c)}$ is
\begin{align} \label{Vol_D}
\text{Vol}&(R_L^{(c)}) = \int_{(x_1,y_1) \in C_1}\cdots \int_{(x_c, y_c) \in C_1} \int_{a_{2c+1}\in S_1}\cdots  \int_{a_{L}\in S_1} \nonumber \\
	& \prod\limits_{\substack{ 1 \leq u < v \leq L \\ a_1=x_1+jy_1, \cdots, a_c=x_c+jy_c} } \frac{\left|  (a_v-a_u) \right| }{c! (L-2c)!}
	dx_1 dy_1 \cdots da_L,
\end{align}
where $C_1=\{(x,y) \mid x^2+y^2 \leq 1\}, S_1 =\{x \mid -1 \leq x \leq 1 \}$.
\end{theorem}
\begin{proof}
Let $J(a_1,\cdots,a_L)$  be the value of the left-hand side of Equation (\ref{Jacob1}), which is a polynomial of $a_1,\cdots,a_L$.
For any positive integers $k,u$ and $v$, it is clear that
$$
\frac{\partial \lambda_k}{\partial a_u} =
 \sum\limits_{\substack{ 1 \leq \ell_1 <  \cdots < \ell_{k-1} \leq L,  \\
 					\ell_1, \cdots, \ell_{k-1} \neq u}} a_{\ell_1} \cdots a_{\ell_{k-1}}.
$$
Thus $J$ changes sign under any transposition of the $a_u$ and $a_v$ by properties of the determinant, i.e. $J$ is an alternating polynomial of $a_1,\cdots,a_L$.
It implies that $J$ is divisible by the Vandermonde polynomial $V(a_1,\cdots,a_L) = \prod\limits_{1 \leq u < v \leq L} (a_v-a_u)$ \cite{Num_Theory-08}.
Furthermore, both $V(a_1,\cdots,a_L)$ and $J(a_1,\cdots,a_L)$ are homogeneous polynomials of degree $(L-1)L/2$, and the coefficients of the term $a_{L-1}^{L-1}a_{L-2}^{L-2}\cdots a_{2}$ are both $1$. Therefore, we  obtain $V(a_1,\cdots,a_L) = J(a_1,\cdots,a_L)$.

In order to compute the volume of $R_L^{(c)}$, consider the space
$C_1^c \times S_1^{L-2c} \subset \mathbb{R}^L$.
Each point $(x_1,y_1$,$\cdots$, $a_{2c+1}$,$\cdots$, $a_L) $ in $R_L^{(c)}$ corresponds to a set of roots $ (x_1+jy_1, x_1-jy_1, \cdots$, $a_{2c+1}$, $\cdots$,$a_L)$ of a stable polynomial in $R_L^{(c)}$, and thus there is a $2^c c!(L-2c)! : 1$ mapping from  $C_1^c \times S_1^{L-2c} $ to $R_L^{(c)}$ (due to the permutation of roots).
Therefore,
\begin{align} \label{eqn1}
\text{Vol}&(R_L^{(c)}) = \int_{(x_1,y_1) \in C_1}\cdots \int_{a_{L}\in S_1}  \frac{1}{2^c c! (L-2c)!}  \nonumber \\
	&\left| \det \left( \frac{\partial [\lambda_1,\lambda_2,\cdots, \lambda_{2c+1}, \cdots, \lambda_L]}{\partial [x_1,y_1\cdots,a_{2c+1},\cdots,a_L]} \right) \right| dx_1 dy_1 \cdots da_L .
\end{align}
Since $a_k=x_k  +jy_k, k=1,\cdots, c$ implies that
$$
\det \left( \frac{\partial [a_1,a_2,\cdots, a_{2c+1}, \cdots, a_L]}{\partial [x_1,y_1\cdots,a_{2c+1},\cdots,a_L]} \right) = (-2j)^c,
$$
 we  obtain Equation (\ref{Jacob1}) by combining Equations (\ref{Vol_D})--(\ref{eqn1}).
\end{proof}

Following Theorem \ref{Thm:Jac-Disc}, it is not difficult to obtain the following result. We omit the proof for brevity.
\begin{corollary} \label{corollary:root_draw}
Generating a sample uniformly from $R_L$ can be done via the following three-step procedure:
\begin{enumerate}
\item Randomly draw
$$c \sim Multinomial \left(\frac{\text{Vol}(R_L^{(0)})}{\text{Vol}(R_L)}, \cdots, \frac{\text{Vol}(R_L^{(\lfloor \frac{L}{2} \rfloor)})}{\text{Vol}(R_L)} \right); $$

\item Generate a random sample $(x_1,y_1$,$\cdots$,$x_c,y_c$, $a_{2c+1}$, $\cdots$, $a_L)$ from $R_L^{(c)}$ according to the (unnormalized) density
\begin{align}
\prod\limits_{\substack{ 1 \leq u < v \leq L \\ a_1=x_1+jy_1, \cdots, a_c=x_c+jy_c} }  |a_v-a_u|. \label{root_density}
\end{align}
\item Obtain $(\lambda_1, \cdots, \lambda_L)$ by computing
$$\lambda_k = (-1)^k \sum\limits_{1 \leq \ell_1 < \cdots < \ell_k \leq L} a_{\ell_1} \cdots a_{\ell_k}, \, k=1,\cdots,L.$$
\end{enumerate}
\end{corollary}

In practical implementations, the parameters of multinomial distribution in the first step require only one-time computation, and the second step can be realized by a sequence of one-dimensional reject samplings.


\begin{figure}[b!]
     \begin{center}
            \includegraphics[width=3.6 in]{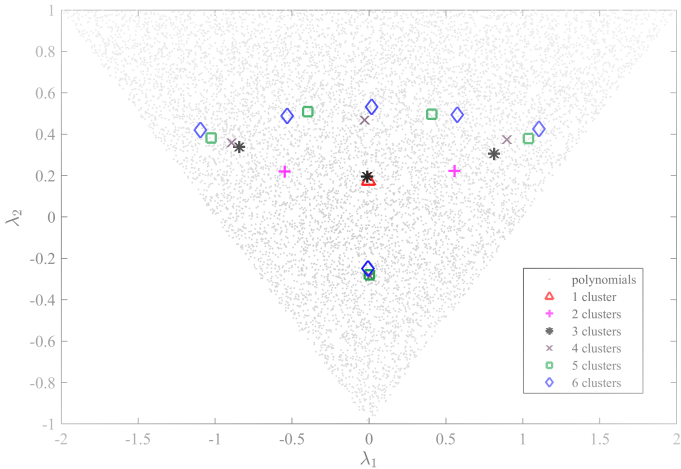}
    \end{center}
    \vspace{0.0 in}
    \caption{The coefficients of 10000 second order stable filters ($z^2+ \lambda_{1}z+\lambda_2 $) that are independently generated from the uniform distribution on $R_2$, and the centers for different number of clusters.}
    \label{Uniform_Poly}
\end{figure}

\begin{figure} [t!]
     \begin{center}
            \includegraphics[width=3.6 in]{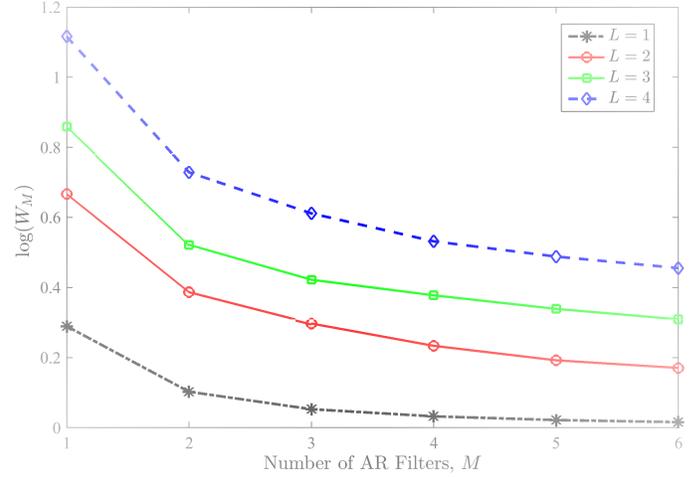}
       \end{center}
       \vspace{-0.0 in}
       \caption{Reference curves for $L = 1, 2, 3,$ and 4 and $M_{\max} = 6$.}
\label{fig:Ref_Curve}
\end{figure}

\subsection{Calculating the Reference Curve}

Based on the  discussions above, the following procedure describes how to derive the reference curve and Gap statistics.
\begin{enumerate}
\item
Generate $F$ uniform random stable filters $\mathfrak{B} = \{\bm \psi_1$, $\bm \psi_2$, $\cdots$, $\bm \psi_F\}$ with a given size $L$, using the technique introduced in Section~\ref{sec:distribution}.

\item
Suppose that $\{1,\cdots,M_{\textrm{max}} \}$ is the candidate set of numbers of modes.
For each $M=1,2,\cdots,M_{\textrm{max}}$, cluster $\mathfrak{B}$   into $M$ disjoint clusters $C_1,C_2,\cdots,C_M$ by minimizing:
\begin{align} \label{objective_filter}
      W_M \triangleq   \min\limits_{\bm \gamma_1,\cdots,\bm \gamma_M } \left\{  \frac{1}{F} \sum_{m=1}^M \sum_{\bm \psi \in C_m} D(\bm \gamma_m, \bm \psi) \right\} +1 ,
 \end{align}
where $D(\bm \gamma_m, \bm \psi)$ has been defined in Section~\ref{sec:distance} and the clusters $C_m$ are similarly defined as in (\ref{clusters}).\footnote{
To make it consistent with the MSPE in (\ref{Pred_Error_of_FilterB}), we put an additional ``1'' on the right hand side of definition (\ref{objective_filter}) compared with (\ref{objective_general}). The ``1'' was used to represent the noise variance $\sigma^2$,  since $\log(\sigma^2)$ becomes a linear term in $\log W_M$ and thus can be negligible. }
For this step, we first generate the matrix whose elements are pairwise distances between sampled filters, and then run the $k$-medoids algorithm \cite{kaufman1987clustering} to approximate the optimum of (\ref{objective_filter}).

\item
Plot the reference curve, which is $\log(W_M)$ against $M$ for $M=1, \cdots,M_{\textrm{max}})$.
We note that the reference curve is model independent.

\item
Plot the empirical curve given the MSPE for $M=1,2,\cdots,M_{\textrm{max}}$, using the observed data, postulated model, and the model fitting approach.
For example, the postulated model in this paper is the mixture of AR introduced in Section \ref{sec:Model}, and the model fitting approach is the EM algorithm.

\item
Finally, the number of AR mixtures that corresponds to the largest gap between the two curves is selected. 

\end{enumerate}

Fig.~\ref{Uniform_Poly} illustrates the sampled coefficients of stable filters of size $2$ randomly generated using the technique described in Section~\ref{sec:distribution}. The centers of the clusters (which are approximated as some of the generated filter samples) obtained using $k$-medoids algorithm are also shown in this figure for $M=1,2,\dots,6$. These centers are calculated based on the average of 20 random instances, each with $1000 $ samples.
Fig.~\ref{fig:Ref_Curve} shows the reference curves for $L=1,2,3$ and 4. Similar to Fig.~\ref{Uniform_Poly}, the reference curves are plotted based on $20$ random instances, each with $1000$ samples.


\section{Numerical Results}
\label{sec:performance}

\begin{table} [t!]
    \vspace{0.1 in}
    \begin{center}
        \includegraphics[width=3.5 in]{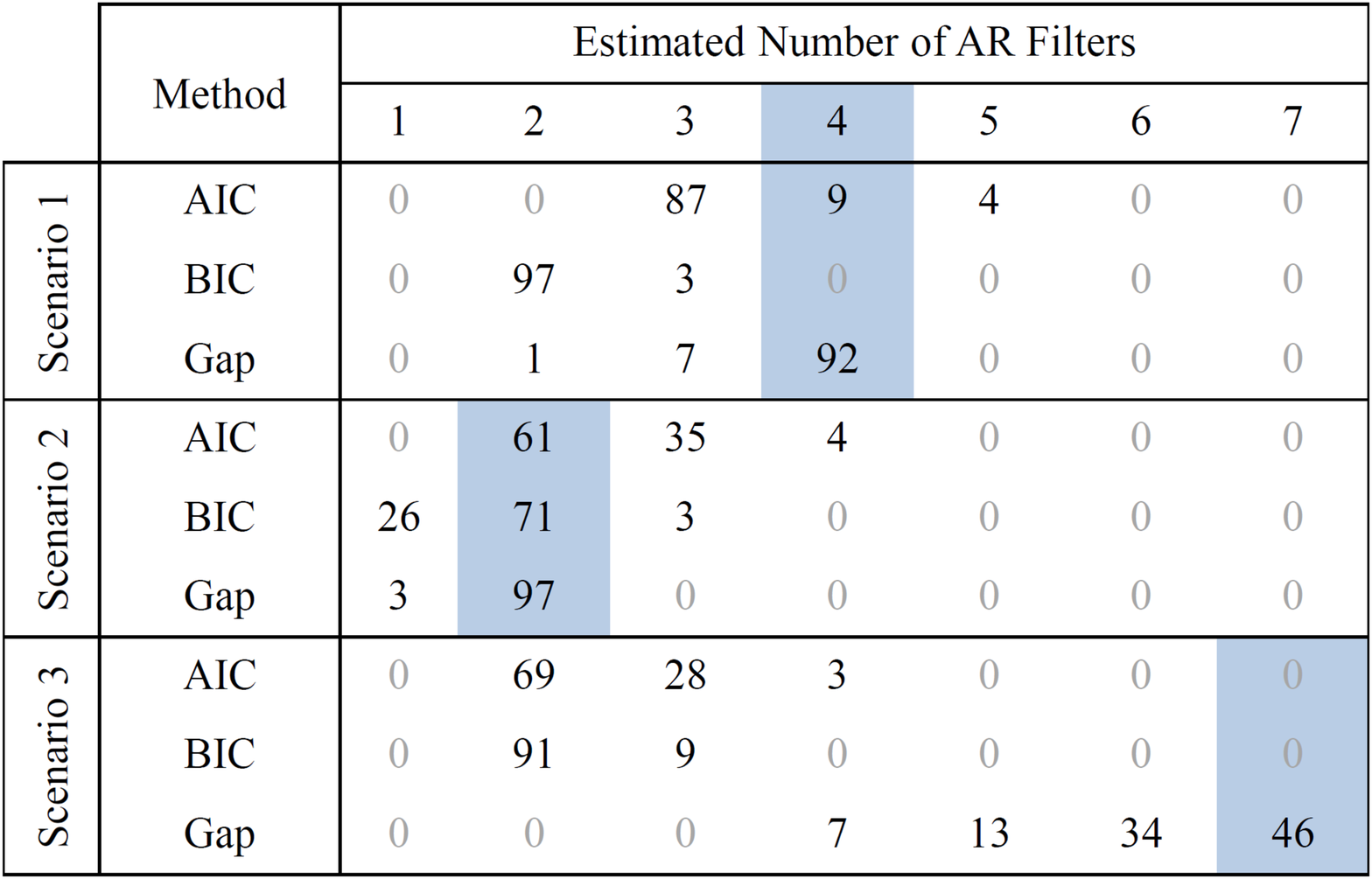}
    \end{center}
    \vspace{-0.0 in}
    \caption{The estimated number of AR filters for three different scenarios using AIC, BIC and Gap statistics (with the true number of filters for each scenario highlighted)}
    \label{tab:results}
\end{table}

We have generated zero-mean time-series with size $N=1400$ using the following three different models to evaluate the performance of the proposed model selection technique.

\emph{Scenario 1}: 4 AR filters with lengths $L=2$. The AR filter at each time-instance is randomly drawn from a multinomial distribution with parameters (0.25, 0.25, 0.25, 0.25).

\emph{Scenario 2}: 2 AR filters with lengths $L=4$. The AR filter at each time-instance is randomly drawn from a Bernoulli distribution with parameters (0.4, 0.6).

\emph{Scenario 3}: 7 AR filters with lengths $L=1$. The time-series is divided into 7 equal parts and only one AR filter is used to draw the data-points in each part.

For each scenario, $100$ time-series of length $1400$ are generated using   uniformly distributed random AR filters. For each time series, EM is run $50$ times with different random initializations to increase the chance that the final estimates are close to the global optimum. Table~\ref{tab:results} shows the estimated number of AR filters using AIC, BIC, and Gap statistics. The true number of filters for each scenario is highlighted. As it can be observed, Gap statistics outperforms AIC and BIC in all of the three scenarios. 
In the third scenario, AIC and BIC are not able to estimate the number of AR filters correctly, and severely underestimate the number of modes. While the Gap statistics finds the correct number of AR filters $46\%$ of the time.

For small $L$  and large $M$, the uniformly generated AR filters are more likely to be close to one another, and thus identifying them as two separate modes becomes more challenging.
Nevertheless, Gap statistics still outperforms AIC and BIC for this scenario.
By increasing $L$, the volume of the space of stable filters $R_L$ explodes and the average distance (defined in (\ref{def_distance})) between two randomly chosen filters becomes larger, which makes the mode separation much easier for large $L$'s.
In that case, the gain of our approach is more pronounced. 


%
%


\section{Conclusions}
\label{sec:conclusion}

In this work, we introduced a new model selection technique based on Gap statistics in order to estimate the number of stable AR mixtures for modeling a given time series.
The Gap statistics was extended to stable filters using a new distance measure between stable AR filters. This distance measure in turn was derived based on mean squared prediction error (MSPE).
We also proposed a method to generate uniform random stable AR filters in the root domain, in order to compute the reference curve. This  may be of some independent interest on its own right.
Simulation results were provided demonstrating the performance of our proposed approach.




\section*{Acknowledgment}
This research was funded by the Defense Advanced Research Projects Agency (DARPA) under grant number W911NF-14-1-0508.

\bibliographystyle{IEEEtran}
\balance
\bibliography{ARMA,Clustering}

\end{document}